\documentclass[11 pt]{article}
\usepackage[utf8]{inputenc}
\usepackage{url}
\usepackage{amsmath}
\usepackage{authblk}
\usepackage{amssymb}
\usepackage{hyperref}
\usepackage{graphics}
\usepackage{mathtools}
\usepackage{caption}
\usepackage{subcaption}

\usepackage{graphicx}
\usepackage{algpseudocode}
\usepackage{algorithm}
\usepackage{afterpage}
\usepackage{array}
\usepackage{collcell}
\usepackage{tcolorbox}
\usepackage{multirow}
\usepackage{amsthm}

\newtheorem{theorem}{Theorem}[section]
\newtheorem{lemma}[theorem]{Lemma}

\usepackage{cite}
\usepackage{booktabs}
\usepackage{appendix}

\usepackage[left=1in,right=1in,top=1.2in,bottom=1.2in,
            footskip=.25in]{geometry}

\usepackage{algpseudocode}
\usepackage{algorithm}
\def\x{{\mathbf x}}
\def\w{{\boldsymbol{\omega}}}

\def\f{{\boldsymbol f}}

\def\y{{\mathbf y}}
\def\e{{\boldsymbol \epsilon}}

\def\R{{\mathbb R}}

\title{Generalization Bound for Diffusion Models using Random Features}
\author[1]{Esha Saha}
\author[1]{Giang Tran}

\affil[1]{Department of Applied Mathematics, University of Waterloo}

\date{}
\begin{document}
\maketitle
\begin{abstract}
Diffusion probabilistic models have been successfully used to generate data from noise. However, most diffusion models are computationally expensive and difficult to interpret with a lack of theoretical justification. Random feature models on the other hand have gained popularity due to their interpretability but their application to complex machine learning tasks remains limited.  In this work,  we present a diffusion model-inspired deep random feature model that is interpretable and gives comparable numerical results to a fully connected neural network having the same number of trainable parameters. Specifically, we extend existing results for random features and derive generalization bounds between the distribution of sampled data and the true distribution using properties of score matching. We validate our findings by generating samples on the fashion MNIST dataset and instrumental audio data.
\end{abstract}

\section{Introduction}
Generative modeling has been successfully used to generate a wide variety of data. Some well-known models are Generative Adversarial Networks \cite{goodfellow2014generative, dinh2017density,gan4}, flow-based models \cite{ho2019flow++,li2023fast,zhang2021diffusion}, autoregressive models \cite{menick2018generating,hoogeboom2021autoregressive,kingma2021variational}, and variational autoencoders \cite{gan4,kingma2019introduction,pu2016variational,xu2017variational}. 
Remarkable results can also be obtained using energy-based modeling and score matching \cite{song2020improved,lm2,sde1}. 
Diffusion models are one such class of generative models that give exemplary performance in terms of data generation. A diffusion probabilistic model (or ``diffusion model") is a parameterized model that is trained using variational inference to generate samples matching the data from input distribution after a finite number of timesteps. The model learns to reverse a diffusion process, which is a fixed Markov chain adding noise to the input data until it is destroyed. If the (Gaussian) noise added is very small, then the transition kernels are also conditional Gaussian distribution which leads to a neural network parameterization of the mean and the variance of the transition kernels. Most of the existing diffusion models are extremely complex and very computationally expensive.  In this paper, we propose a model architecture to bridge the gap between interpretable models and diffusion models. Our main idea is to build a deep random feature model inspired by semi-random features \cite{kawaguchi2018deep} and diffusion models as formulated in \cite{ddpm}.

Our paper is organized as follows. We give a brief background of diffusion models and random features in Section \ref{sec:background} along with the related works. The model architecture along with the algorithm and its theoretical results are given in Section \ref{sec:model}. All experimental results are summarized in Section \ref{sec:results} followed by Section \ref{sec:conc} where we conclude our paper and discuss future directions.

\subsection{Contributions}

We propose a diffusion model-inspired random feature model, that uses semi-random-like features to learn the reverse process in the diffusion models. Our main contributions are given below:
\begin{itemize}
    \item Our proposed model is the first of its kind to combine the idea of random features with generative models. It acts as a bridge between the theoretical and generative aspects of the diffusion model by providing approximation bounds of samples generated by diffusion model-based training algorithms. We show that for each fixed timestep, our architecture can be reduced to a random feature model preserving the properties of interpretability. 
    \item Our numerical results on the fashion MNIST and the audio data validate our findings by generating samples from noise, as well as denoising signals not present in the training data. We show that even with a very small training set, our proposed model can denoise data and generate samples similar to training data.
\end{itemize}

\section{Background and related works}
\label{sec:background}

 We first recall some useful notations and terminologies corresponding to diffusion models. In this paper, we denote $\mathcal{N}(\mathbf{0},\mathbf{I})$ the $d-$ dimensional Gaussian distribution with the zero mean vector and the identity covariance matrix. We say that a function 
$p(\x) = \mathcal{N}(\boldsymbol{\mu},\boldsymbol{\Sigma})$ to mean that $p(\x)$ is the p.d.f. of a random vector $\x$ following the multivariate normal distribution with mean vector $\boldsymbol{\mu}$ and covariance matrix $\boldsymbol{\Sigma}$.

A diffusion model consists of two Markov chains: a forward process and a reverse process. The goal of the forward process is to degrade the input sample by gradually adding noise to the data over a fixed number of timesteps. The reverse process involves learning to undo the added-noise steps using a parameterized model. Knowledge of the reverse process helps to generate new data starting with a random noisy vector followed by sampling through the reverse Markov chain \cite{koller2009probabilistic,yang2022diffusion}. 

 \subsection{Forward Process }
The forward process degrades the input data such that $q(\x_K)\approx \mathcal{N}(\mathbf{0},\mathbf{I})$. More precisely, let $\mathbf{x}_0\in\R^d$ be an input from an unknown distribution with p.d.f. $q(\x_0)$. Given a variance schedule $0<\beta_1<\beta_2<...<\beta_K<1$, the forward process to obtain a degraded sample for a given timestep is defined as:
\begin{equation}\label{eqn:forward}
    \x_{k+1} = \sqrt{1-\beta_{k+1}}\x_{k} + \sqrt{\beta_{k+1}}\mathbf{\e}_{k},\,\,\,\,\, 
    \text{where}\,\, \e_{k}\sim\mathcal{N}(\mathbf{0},\mathbf{I})\,\,\text{and}\,\,k=0,\hdots,K-1.
\end{equation}
The forward process generates a sequence of random variables $\x_1,\x_2,...,\x_K$ with conditional distributions 
\begin{equation}\label{eqn:forward_xk}
    q(\x_{k+1}|\x_{k})= \mathcal{N}(\x_{k+1};\sqrt{1-\beta_{k+1}}\x_{k},\beta_{k+1} \mathbf{I}).
    \end{equation}
Let $\alpha_{k} = 1-\beta_{k}$ for $k=1,\ldots, K$ and $\overline\alpha_{k} = \prod\limits_{i=1}^{k}\alpha_i$. 
Using the reparameterization trick, we can obtain $\x_{k+1}$ at any given time $k\in \{0,\ldots, K-1\}$ from $\x_0$:
\begin{equation}
\begin{aligned}
    \x_{k+1} &= \sqrt{\alpha_{k+1}}\,\,\x_{k} + \sqrt{1- \alpha_{k+1}}\,\,\mathbf{\e}_{k}\\
  &=\sqrt{\alpha_{k+1}}\left(\sqrt{\alpha_{k}}\x_{k-1} + \sqrt{1 - \alpha_{k}}\,\mathbf{\e}_{k-1}\right) + \sqrt{1- \alpha_{k+1}}\mathbf{\e}_{k}\\
  &=\sqrt{\alpha_{k+1}\alpha_{k}}\,\,\x_{k-1} + \sqrt{1-\alpha_{k+1}\alpha_{k}}\,\,\widetilde{\e}_{k-1}\\
  & \vdots\\  
    & = \sqrt{\prod\limits_{i=1}^{k+1}\alpha_i}\,\, \x_0 + \sqrt{1-\prod\limits_{i=1}^{k+1}\alpha_i}\,\,\widetilde{\e}_0,
    \end{aligned}
\end{equation}
where $\widetilde{\e}_{i}\sim \mathcal{N}(\mathbf{0},\mathbf{I})$ for $i=0,\ldots,k-1$.
Therefore, the conditional distribution $q(\x_{k+1}|\x_{0})$ is 
\begin{equation}\label{eqn:forward_x0}
q(\x_{k+1}|\x_{0}) = \mathcal{N}(\x_{k+1};\sqrt{\overline\alpha_{k+1}}\,\x_0,(1-\overline\alpha_{k+1})\mathbf{I}),
\end{equation}
At $k=K$, we have
\[\x_K = \sqrt{\overline\alpha_K}\,\x_0 + \sqrt{1-\overline\alpha_K}\,\e,\]
where $\e\sim\mathcal{N}(\mathbf{0},\mathbf{I})$.
Since $0<\beta_1<...<\beta_K<1$,  $0<\overline\alpha_K <\alpha_1^K<1$. Therefore, 
$\lim\limits_{K\rightarrow\infty}\overline\alpha_K=0$. Hence, $q(\x_K) = \int q(\x_K|\x_0)q(\x_0)d\x_0\approx\mathcal{N}(\mathbf{0},\mathbf{I})$, i.e., as the number of timesteps becomes very large, the distribution $q(\x_{K})$ will approach the Gaussian distribution with mean $\mathbf{0}$ and covariance $\mathbf{I}$.

\subsection{Reverse Process }

The reverse process aims to generate data from the input distribution by sampling from $q(\x_K)$ and gradually denoising for which one needs to know the reverse distribution $q(\x_{k-1}|\x_{k})$. In general, computation of $q(\x_{k-1}|\x_{k})$ is intractable without the knowledge of $\x_0$. Therefore, we condition the reverse distribution on $\x_0$ in order to obtain the mean and variance for the reverse process. More precisely, using Baye's rule, we have:
\begin{equation*}
q(\x_{k-1}|\x_k,\x_0) = q(\x_k|\x_{k-1},\x_0)\, \dfrac{q(\x_{k-1},\x_0)}{q(\x_k,\x_0)}\, \dfrac{q(\x_0)}{q(\x_0)} = q(\x_k|\x_{k-1},\x_0)\,\dfrac{q(\x_{k-1}|\x_0)}{q(\x_k|\x_0)}.\end{equation*}
Using Equation \eqref{eqn:forward_xk}, Equation \eqref{eqn:forward_x0}, we have
\begin{multline*}
q(\x_{k-1}|\x_k,\x_0) =\dfrac{1}{\sqrt{(2\pi\beta_k)^d}}\exp\left(-\dfrac{1}{2}\dfrac{(\x_k - \sqrt{\alpha_{k}}\x_{k-1})^T(\x_k - \sqrt{\alpha_{k}}\x_{k-1})}{\beta_k}\right)\\
\cdot 
\dfrac{1}{\sqrt{(2\pi(1-\overline{\alpha}_{k-1}))^d}}\exp\left(-\dfrac{1}{2}\dfrac{(\x_{k-1} - \sqrt{\overline{\alpha}_{k-1}}\x_{0})^T(\x_{k-1} - \sqrt{\overline{\alpha}_{k-1}}\x_{0})}{1-\overline{\alpha}_{k-1}}\right)\\
\cdot\left(\sqrt{(2\pi(1-\overline{\alpha}_{k}))^d}\right)\exp\left(\dfrac{1}{2}\dfrac{(\x_{k} - \sqrt{\overline{\alpha}_{k}}\x_{0})^T(\x_{k} - \sqrt{\overline{\alpha}_{k}}\x_{0})}{1-\overline{\alpha}_{k}}\right)\\
= \dfrac{\sqrt{(1-\overline{\alpha}_{k})^d}}{\sqrt{(2\pi\beta_k(1-\overline{\alpha}_{k-1}))^d}}\exp\Bigg\{-\dfrac{1}{2}\Bigg[ \dfrac{\x_k^T\x_k-2\sqrt{\alpha_{k}}\x_k^T\x_{k-1}+\alpha_{k}\x_{k-1}^T\x_{k-1}}{\beta_k}\\
+ \dfrac{\x_{k-1}^T\x_{k-1} - 2\sqrt{\overline{\alpha}_{k-1}}\x_{0}^T\x_{k-1}+ \overline{\alpha}_{k-1}\x_{0}^T\x_0}{1-\overline{\alpha}_{k-1}}
-\dfrac{\x_{k}^T\x_k - 2\sqrt{\overline{\alpha}_{k}}\x_k^T\x_{0}+ \overline{\alpha}_k\x_0^T\x_0}{1-\overline{\alpha}_{k}}\Bigg]\Bigg\}\\
=  \dfrac{\sqrt{(1-\overline{\alpha}_{k})^d}}{\sqrt{(2\pi\beta_k(1-\overline{\alpha}_{k-1}))^d}}\exp\Bigg\{-\dfrac{1}{2}\dfrac{1-\overline{\alpha}_k}{\beta_k (1-\overline\alpha_{k-1})} \x_{k-1}^T \x_{k-1} +  \bigg(\dfrac{\sqrt{\alpha_k}}{\beta_k}\x_k^T + \dfrac{\sqrt{\overline\alpha_{k-1}}}{1- \overline\alpha_{k-1}}\x_0^T) \x_{k-1} + \text{terms}(\x_k,\x_0)\Bigg\}
\end{multline*}
\begin{equation}
= \dfrac{\sqrt{(1-\overline{\alpha}_{k})^d}}{\sqrt{(2\pi\beta_k(1-\overline{\alpha}_{k-1}))^d}} \exp\left(-\frac{1}{2}\dfrac{(\x_{k-1} - \Tilde{\boldsymbol{\mu}}_k)^T(\x_{k-1} - \Tilde{\boldsymbol{\mu}}_k)}{\Tilde{\beta}_k}\right)
\end{equation}

Thus we see that the probability density function of the reverse distribution conditioned on $\x_0$ is also a Gaussian distribution with mean vector $\Tilde{\boldsymbol{\mu}}_k$ and covariance matrix $\Tilde{\beta}_k\mathbf{I}$ which are given by,
 \begin{equation}\label{eq:mu}
\Tilde{\boldsymbol{\mu}}_k = \dfrac{\sqrt{\alpha_k}(1-\overline{\alpha}_{k-1})}{1-\overline{\alpha}_k}\x_k + \dfrac{\sqrt{\overline{\alpha}_{k-1}}\beta_k}{1-\overline{\alpha}_k}\x_0\,\,\text{and}\,\,
     \Tilde{\beta}_k = \dfrac{1-\overline{\alpha}_{k-1}}{1-\overline{\alpha}_k}\beta_k
\end{equation}
Thus the reverse distribution conditioned on $\x_0$ is $q(\x_{k-1}|\x_{k},\x_0) = \mathcal{N}(\Tilde{\boldsymbol{\mu}}_k,\Tilde{\beta}_k\mathbf{I})$, where $\Tilde{\boldsymbol{\mu}}_k, \Tilde{\beta}_k$ are obtained above. Our aim is to learn the reverse distribution from the obtained conditional reverse distribution. From Markovian theory, if $\beta_{k}$'s are small, the reverse process is also Gaussian \cite{sohl2015deep}. 
Suppose $p_{\theta}(\x_{k-1}|\x_{k})$ be the learned reverse distribution, then Markovian theory tells us that $p_{\theta}(\x_{k-1}|\x_{k}) = \mathcal{N}(\boldsymbol{\mu}_{\theta}(\x_k,k),\boldsymbol{\Sigma}_{\theta}(\x_{k},k))$, where $\boldsymbol{\mu}_{\theta}(\x_k,k)$ and $\boldsymbol{\Sigma}_{\theta}(\x_{k},k)$ are the learned mean vector and variance matrix respectively. 
Since the derived covariance matrix $\Tilde{\beta}_k\mathbf{I}$ for conditional reverse distribution is constant, $\boldsymbol{\Sigma}_{\theta}(\x_{k},k)$ need not be learnt. In \cite{ddpm}, the authors show that choosing $\boldsymbol{\Sigma}_{\theta}(\x_{k},t)$ as $\beta_{k} \mathbf{I}$ or $\Tilde{\beta}_{k} \mathbf{I}$ yield similar results and thus we fix $\boldsymbol{\Sigma}_{\theta}(\x_{k},k) = \beta_{k} \mathbf{I}$ for simplicity. Furthermore, since $\x_k$ is also available as input to the model, the loss function derived in \cite{ddpm} as a KL divergence between $q(\x_{k-1}|\x_k,\x_0)$ and $p_{\theta}(\x_{k-1}|\x_{k})$ can be simplified as
\begin{equation}\label{eq:kldiv}
    D_{KL}(q(\x_{k-1}|\x_{k},\x_0)\| p_{\theta}(\x_{k-1}|\x_{k})= \mathbb{E}_q\left[\frac{1}{2\beta_k^2}\|\Tilde{\boldsymbol{\mu}}_k(\x_{k},\x_0) - \boldsymbol{\mu}_{\theta}(\x_{k},k)\|^2 \right] +\text{const}
\end{equation}
where for each timestep $k$, $\boldsymbol{\Tilde{\mu}}_k(\x_k,\x_0)$ denotes the mean of the reverse distribution conditioned on $\x_0$ i.e., $q(\x_{k-1}|\x_k,\x_0)$ and $\boldsymbol{\mu}_{\theta}(\x_k,k)$ denotes the learnt mean vector.
Thus, the above equation predicts the mean of the reverse distribution when conditioned on $\x_0$. Substituting $\x_0= \dfrac{1}{\sqrt{\overline{\alpha}_k}}(\x_k - \sqrt{1-\overline{\alpha}_k}\boldsymbol{\epsilon}_k)$ in Eq. \eqref{eq:mu} we can obtain $\Tilde{\boldsymbol{\mu}}(\x_k,k) = \dfrac{1}{\sqrt{\alpha_k}}\left(\x_k - \dfrac{\beta_k}{\sqrt{1-\overline{\alpha}_k}}\e\right)$. Further, since $\x_k$ is known, we can use the formula for $\boldsymbol{\mu}_{\theta}(\x_k,k) = \dfrac{1}{\sqrt{\alpha_k}}\left(\x_k - \dfrac{\beta_k}{\sqrt{1-\overline{\alpha}_k}}\e_{\theta}(\x_k,k)\right)$. 
We can simplify Eq. \eqref{eq:kldiv} as:

\begin{equation}\label{eq:mainloss}
   \mathbb{E}_{k,\x_0,\boldsymbol{\epsilon}}\left[\dfrac{1}{2\alpha_k(1-\overline{\alpha}_k)}\|\e - \e_{\theta}(\x_k,k)\|^2\right].
\end{equation}
where $\e_{\theta}$ now denotes a function approximator intended to predict the noise from $\x_k$. 
The above results show that we can either train the reverse process mean function approximator $\boldsymbol{\mu}_{\theta}$ to predict $\Tilde{\boldsymbol{\mu}}_k$ or modify using its parameterization to predict $\e$. In our proposed algorithm, we choose to use the loss function from Eq. \eqref{eq:mainloss} since it is one of the simplest forms to train and understand. This formulation of DDPM also helps us to harness the power of SDEs in diffusion models through its connection to DSMs \cite{block2020generative} .

\subsection{DDPM and DSM}
We can also apply the DDPM algorithm for score matching by formulating the DDPM objective as a DSM objective.
\begin{align}\label{eq:ddpmtodsm}
    L_{\text{DDPM}} &= \mathbb{E}_{k,\x_0,\boldsymbol{\epsilon}}\left[\dfrac{1}{2\alpha_k(1-\overline{\alpha}_k)}\|\e - \e_{\theta}(\x_k,k)\|^2\right]\\
    &=\mathbb{E}_{k,\x_0,\boldsymbol{\epsilon}}\left[\dfrac{1}{2\alpha_k(1-\overline{\alpha}_k)}\Bigg\|\dfrac{\x_k - \sqrt{\overline{\alpha}_k}\x_0}{\sqrt{1-\overline{\alpha}_k}} - \e_{\theta}(\x_k,k)\Bigg\|^2\right]\\
    & = \mathbb{E}_{k,\x_0,\x_k}\left[\dfrac{1}{2\alpha_k(1-\overline{\alpha}_k)}\Bigg\|\dfrac{\x_k - \sqrt{\overline{\alpha}_k}\x_0}{1-\overline{\alpha}_k}\sqrt{1-\overline{\alpha}_k} -\dfrac{\sqrt{1-\overline{\alpha}_k}}{\sqrt{1-\overline{\alpha}_k}} \e_{\theta}(\x_k,k)\Bigg\|^2\right]\\
    & = \mathbb{E}_{k,\x_0,\x_k}\left[\dfrac{1}{2\alpha_k}\Bigg\|-\nabla_{\x_k}\log q(\x_k|\x_0) -\dfrac{1}{\sqrt{1-\overline{\alpha}_k}} \e_{\theta}(\x_k,k)\Bigg\|^2\right]\\
    &=\mathbb{E}_{k,\x_0,\x_k}\left[\dfrac{1}{2\alpha_k}\Bigg\|s_{\theta}(\x_k,k) - \nabla_{\x_k}\log q(\x_k|\x_0)\Bigg\|^2\right] = L_{\text{DSM}}
\end{align}
where $s_{\theta}(\x_k,k) =  -\dfrac{1}{\sqrt{1-\overline{\alpha}_k}} \e_{\theta}(\x_k,k)$. The above formulation is known as  denoising score matching (DSM), which is equivalent to the objective of DDPM. Furthermore, the objective of DSM is also related to the objective of score based generative models using SDEs \cite{song2019generative}. We briefly discuss the connection between diffusion models, SDEs and DSM in the upcoming section \cite{song2019generative,yang2022diffusion,ddpm}.

\subsection{Diffusion Models and SDEs}
The forward process can also be generalized to stochastic differential equations (SDEs) if infinite time steps or noise levels are considered (SDEs) as proposed in \cite{yang2022diffusion}. To formulate the forward process as an SDE, let $t=\frac{k}{K}$ and define functions $\x(t),\beta(t)$ and $\e(t)$ such that $\x(\frac{k}{K}) = \x_k$, $\beta(\frac{k}{K}) = K\beta_k$ and $\e(\frac{k}{K}) = \e_k$. Note that in the limit $K\rightarrow\infty$, we get $t\in [0,1]$. 

Using the derivations from \cite{yang2022diffusion}, the forward process can be written as a SDE of the form,
\begin{align}\label{eq:finalsde}
   d\x &=\dfrac{-\beta(t)}{2}\x dt + \sqrt{\beta(t)}d\mathbf{w}
\end{align}
where $\mathbf{w}$ is the standard Wiener process. 
The above equation now is in the form of an SDE
\begin{equation}
    d\x = f(\x,t)dt + g(t)d\mathbf{w}
\end{equation}
where $f(\x,t)$ and $g(t)$ are diffusion and drift functions of the SDE respectively, and $\mathbf{w}$ is a standard Wiener process. The above processed can be reversed by solving the reverse-SDE,
\begin{equation}
    d\x = [f(\x,t)-g(t)^2\nabla_{\x(t)}\log q(\x(t))]dt + g(t)d\mathbf{\overline{w}}
\end{equation}
where $\overline{\mathbf{w}}$ is a standard Wiener process backwards in time, and $dt$ denotes an infinitesimal negative time step and $q(\x(t))$ is the marginal distribution of $\x(t)$. Note that in particular for Eq.\eqref{eq:finalsde}, the reverse SDE will be of the form
\begin{equation}
    d\x = \left[\dfrac{\beta(t)}{2} - \beta(t)\nabla_{\x(t)}\log q(\x(t))\right]dt + \sqrt{\beta(t)}d\mathbf{\overline{w}}
\end{equation}

The unknown term $\nabla_{\x(t)}\log q(\x(t))$ is called the score function and is estimated by training a parameterized model $s_\theta(\x(t),t)$ via minimization of the loss given by 
\begin{equation}\label{eq:loss_or}
\mathbb{E}_{q(\x(t))}\left[\frac{1}{2}\Big\|s_{\theta}(\x(t),t)- \nabla_{\x(t)}\log q(\x(t)) \Big\|_2^2\right]
\end{equation}

Note that since $q(\x(0))$ is unknown, therefore the distribution $q(\x(t))$ and subsequently the score function $\nabla_{\x(t)}\log q(\x(t))$ are also unknown. Referring to results from \cite{dsm3}, we see that the loss in Eq. \eqref{eq:loss_or} is equivalent to the denoising score matching (DSM) objective given by,
\begin{equation}\label{eq:sgm1}
\mathbb{E}_{q(\x(t),\x(0))}\left[\frac{1}{2}\Big\|s_{\theta}(\x(t),t)- \nabla_{\x(t)}\log q(\x(t)|\x(0)) \Big\|_2^2\right].
\end{equation}

We can see that the above objective is the same as the objective of DSM in the discrete setting. Apart from the success of score based models using SDEs, an additional advantage of formulating the diffusion model using SDEs is the theoretical analysis based on results from SDEs. In our paper, we aim to use this connection to build a theoretical understanding of our proposed model.

While there are remarkable results for improving training and sampling for diffusion models, little has been explored in terms of the model architectures. Since distribution learning and data generation is a complex task, it is unsurprising that conventional diffusion models are computationally expensive. From previous works in \cite{ddpm,yang2022diffusion,zhang2021diffusion}, U-Net (or variations on U-Net combined with ResNet, CNNs, etc.) architecture is still the most commonly used model for diffusion models. U-Nets not only preserve the dimension of input data, they also apply techniques such as downsampling using convolution which helps to learn the features if the input data. However, all these architectures have millions of parameters making training (and sampling) cumbersome. 
An alternative approach to reduce the complexity of machine learning algorithms is to use a random feature model (RFM) \cite{rahimi2007random, rahimi2008uniform} for approximating the kernels using a randomized basis. RFMs are derived from kernel-based methods which utilize a pre-defined nonlinear function basis called kernel $K(\x,\y)$. 
 From the neural network point of view, an RFM is a two-layer network with a fixed single hidden layer sampled randomly \cite{rahimi2007random, rahimi2008uniform}. 
Not only do random feature-based methods give similar results to that of a shallow network, but the model in itself is also interpretable which makes it a favorable method to use. Some recent works which use random features for a variety of tasks are explored in \cite{saha2022harfe,spade4, performer,choromanski2021hybrid,richardson2022srmd,nelsen2021random}. However, random features can lack expressibility due to their structure and thus we aim to propose an architecture that can be more flexible in learning yet retaining properties of random features. Inspired by semi-random features \cite{kawaguchi2018deep}, DDPM \cite{ddpm} and the idea of building deep random features, our diffusion random feature model serves as a potential alternative to the existing diffusion model architectures.

\section{Algorithm and Theory}\label{sec:model}
\label{headings}
Our proposed model is a diffusion model inspired random feature model. The main idea of our model is to build an interpretable deep random feature architecture for diffusion models. Our work is inspired by \textit{denoising diffusion probabilistic model } (DDPM) proposed in \cite{ddpm} and semi-random features proposed in \cite{kawaguchi2018deep}. Let $\x_0\in\R^d$ be the input data belonging to an unknown distribution $q(\x_0)$. Let $K$ denote the total number of timesteps in which the forward process is applied. Suppose $N$ is the number of features. For each timestep $k$, we build a noise predictor function $p_{\theta}$ of the form
\begin{equation}\label{eq:DRFM}
    p_{\theta} (\x_k,k) =
    (\sin(\x_k^T \mathbf{W} + \mathbf{b}^T)
    \odot 
    \cos (\boldsymbol{\tau}_k^T\boldsymbol{\theta}^{(1)}))\boldsymbol{\theta}^{(2)}, 
\end{equation}
where $\x_k\in\R^{d}$, $\mathbf{W}\in\R^{d\times N}$, $\mathbf{b} =\begin{bmatrix} b_1&\ldots& b_N\end{bmatrix}^T\in\R^N$, $\boldsymbol{\theta}^{(1)} = (\theta_{ki}^{(1)})\in\R^{K\times N}$, $\boldsymbol{\tau}_k\in\R^{K}$, $\boldsymbol{\theta}^{(2)} = (\theta_{ij}^{(2)})\in\R^{N\times d}$, and $\odot$ denotes element-wise multiplication. The vector $\boldsymbol{\tau}_k \, (k\geq 1)$ is a one-hot vector with the position of one corresponding to the timestep $k$. The motivation to use trainable weights corresponding to the time parameter is twofold: first, we want to associate importance to the timestep being used when optimizing the weights; secondly, we aim to build a deep random feature model layered through time. The inspiration for using cosine as an activation comes from the idea of positional encoding used for similar tasks. In general, positional encoding remains fixed, but for our method, we wish to make the weights associated with timestep random and trainable. This is done so that the model learns the noise level associated with the timestep.
Our aim is to train the parameters $\boldsymbol{\theta} =\{\boldsymbol{\theta}^{(1)},\boldsymbol{\theta}^{(2)}\}$ while $\mathbf{W}$ and $\mathbf{b}$ are randomly sampled and fixed. The model is trained using Algorithm \ref{alg:cap}.
\begin{figure}
    \centering
    \includegraphics[width = 4in, height=1.5in]{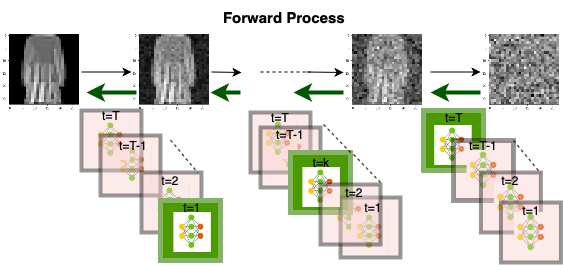}
    \caption{Representation of DRFM. The green boxes denote the random feature layer which is active corresponding to the timestep selected, while the other remain fixed.}
    \label{fig:DRFM_fig}
\end{figure}
\begin{algorithm}[ht!]
\caption{Training and sampling using DRFM}\label{alg:cap}
\begin{algorithmic}[1]
\Require Sample $\x_0\sim q(\x_0)$ where $q$ is an unknown distribution, variance schedule $\beta = \{\beta_1,...,\beta_K\}$ such that $0<\beta_1<\beta_2<...<\beta_K<1$, random weight matrix $\mathbf{W} = [w_{ij}]$ and bias vector $\mathbf{b}$ sampled from a distribution $\rho$ and total number of epochs \texttt{epoch}.
\Ensure 
\Statex{}
\Statex{\underline{Training}}
\Statex{}
\State Choose random timestep $k\in\{1,2,...,K\}$ and build vector $\boldsymbol{\tau}_k = [0,...0,1,0,...,0]^T$ where 1 is in $k^{th}$ position.

\State Define the forward process for $k=1,2,...,K$ as
    \begin{equation*}
        \x_k = \sqrt{1-\beta_k}\x_{k-1} + \sqrt{\beta_k}\e_k
    \end{equation*}
    where $\e_k\sim \mathcal{N}(\mathbf{0},\mathbf{I})$.
\For {$j$ in \texttt{epochs}}
\State{$k\sim\mathcal{U}\{1,2,...,K\} $}.
\State{Define $\boldsymbol{\tau}_k$ as in line 1.}
\State $p_{\theta}(\x_k,k)\gets  
    (\sin(\x_k^T \mathbf{W} + \mathbf{b})
    \odot 
    \cos (\boldsymbol{\tau}_k^T\boldsymbol{\theta^{(1)}}))\boldsymbol{\theta^{(2)}} $
\State Update $\boldsymbol{\theta} = [\boldsymbol{\theta}^{(1)},\boldsymbol{\theta}^{(2)}]$ by minimizing the loss $L = \dfrac{1}{K}\displaystyle\sum\limits_{k=1}^K\Big\|\e_k - p_{\theta
    }(\x_k,k)\Big\|_2^2$.
\EndFor
\Statex
\Statex{\underline{Sampling}}
\Statex
\State{Sample a point 
    $\x_K\sim\mathcal{N}(\mathbf{0},\mathbf{I})$}
\For{$k=K-1,...,1$}
\State{Sample $\e\sim\mathcal{N}(\mathbf{0},\mathbf{I})$}
\State{$\Tilde{\x}_{k-1} =  \dfrac{1}{\sqrt{1-\beta_k}}\left(\x_k - \dfrac{\sqrt{\beta_k}}{\sqrt{1-\prod\limits_{i=1}^k (1-\beta_i)}}p_{\theta}(\x_k,k) ) \right) + \beta_k\e$}
\EndFor
\Statex{\textbf{Output:} Generated sample $\Tilde{\x}_0$}
\end{algorithmic}
\end{algorithm}

\subsection{Theoretical Results}

We provide theoretical results corresponding to our proposed model. We first formulate our proposed model as a time-dependent layered random features model, followed by the proof of obtaining sample generation bounds. The obtained bounds help to prove that DRFM is capable of generating samples from the distribution on which it was trained using the results from \cite{chen2022sampling}. 

For a fixed timestep $k$, Eq. \ref{eq:DRFM} can be written as:

\begin{equation}\label{eq:rf}
\begin{aligned}
    p_{\theta}(\x_k,k) &= (\sin(\x_k^T\mathbf{W} + \mathbf{b}^T)\odot\cos(\boldsymbol{\tau}_k\boldsymbol{\theta}^{(1)}))\boldsymbol{\theta}^{(2)}\\
    &=\left(\sin\left([y_1,...,y_d]\begin{bmatrix}
        \omega_{11}&...&\omega_{1N}\\
        \vdots&...&\vdots\\
        \omega_{d1}&...&\omega_{dN}
    \end{bmatrix}+\begin{bmatrix}
        b_1\\
        \vdots\\
        b_N
    \end{bmatrix}\right)\odot \begin{bmatrix}
        \cos(\theta_{k1}^{(1)}) &...& \cos(\theta_{kN}^{(1)})
    \end{bmatrix}\right)\boldsymbol{\theta}^{(2)}\\
    &= \sin(\x_k^T\mathbf{W} + \mathbf{b}^T)\begin{bmatrix}
        \cos(\theta_{k1}^{(1)})\theta_{11}^{(2)}&...&\cos(\theta_{k1}^{(1)})\theta_{1d}^{(2)}\\
        \vdots & ...&\vdots\\
        \cos(\theta_{kN}^{(1)})\theta_{N1}^{(2)}&...&\cos(\theta_{kN}^{(1)})\theta_{Nd}^{(2)}
    \end{bmatrix}.
\end{aligned}
\end{equation}
 For each $j=1,...,N$, let $a_j = \sin\left(\displaystyle\sum\limits_{i=1}^d y_i\omega_{ij}+b_j\right)$. Note that $a_j$'s are fixed. Then Eq. \eqref{eq:rf} becomes,
\begin{align}
   p_{\theta}(\y,k) &= \begin{bmatrix}
        a_1 \cos(\theta_{k1}^{(1)}) & ... & a_N\cos(\theta_{kN}^{(1)})
    \end{bmatrix}\begin{bmatrix}
        \theta_{11}^{(2)}&...&\theta_{1d}^{(2)}\\
        \vdots & ...&\vdots\\
        \theta_{N1}^{(2)}&...&\theta_{Nd}^{(2)}
    \end{bmatrix}\\
    & = \begin{bmatrix}
        a_1 & ... & a_N
    \end{bmatrix}\begin{bmatrix}
        \cos(\theta_{k1}^{(1)})\theta_{11}^{(2)}&...&\cos(\theta_{k1}^{(1)})\theta_{1d}^{(2)}\\
        \vdots & ...&\vdots\\
        \cos(\theta_{kN}^{(1)})\theta_{N1}^{(2)}&...&\cos(\theta_{kN}^{(1)})\theta_{Nd}^{(2)}
    \end{bmatrix}
\end{align}

 Thus, for a fixed time $k$, our proposed architecture is a random feature model with a fixed dictionary having $N$ features denoted by $\phi(\langle \x_k,\boldsymbol{\omega}_i)\rangle = \sin(\x_k^T\boldsymbol{\omega}_i+b_i),\, \forall i=1,\ldots,N$ and learnt coeeficients $\mathbf{C} = (c_{ij})\in\mathbb{R}^{N\times d}$ whose entries are  

 $c_{ij} = \cos(\boldsymbol{\theta}^{(1)}_{ki})\boldsymbol{\theta}^{(2)}_{ij},\,\,\forall i=1,\ldots, N; j=1,\ldots,d.$

 As depicted in Figure \ref{fig:DRFM_fig}, DRFM can be visualized as $K$  random feature models stacked up in a column. Each random feature model has associated weights corresponding to its timestep which also gets optimized implicitly while training. The reformulation of DRFM into multiple random feature models leads us to our first Lemma stated below. We show that the class of functions generated by our proposed model is the same as the class of functions approximated by random feature models.

\begin{lemma}\label{lemma:equality}
    Let $\mathcal{G}_{k,\w}$ denote the set of functions that can be approximated by DRFM at each timestep $k$ defined as
    \begin{equation}
    \mathcal{G}_{k,\w} = \left\{\boldsymbol{g}(\x) = \sum\limits_{j=1}^N\cos(\boldsymbol{\theta}_{kj}^{(1)})\boldsymbol{\theta}_j^{(2)}\phi(\x_k^T\boldsymbol{\omega}_j)\Bigg | \Big\|\boldsymbol{\theta}_j^{(2)}\Big \|_{\infty}\leq \dfrac{C}{N}\right\}
\end{equation} 
Then for a fixed $k$ and $\mathcal{F}_{\w}$ defined in Eq. \eqref{eq:fw}, $\mathcal{G}_{k,\w}=\mathcal{F}_{\w}$.
\end{lemma}

\begin{proof}
    The above equality can be proved easily. \\
    Fix $k$. Consider $\boldsymbol{g}(\x)\in\mathcal{G}_{k,\w}$, then $\boldsymbol{g}(\x) = \displaystyle\sum\limits_{j=1}^N\cos(\boldsymbol{\theta}_{kj}^{(1)})\boldsymbol{\theta}_j^{(2)}\phi(\x_k^T\boldsymbol{\omega}_j)$. Clearly, $\boldsymbol{g}(\x)\in\mathcal{F}_{\w}$ as $\Big\|\cos(\boldsymbol{\theta}_{kj}^{(1)})\boldsymbol{\theta}_j^{(2)}\Big \|_{\infty}\leq\Big\|\boldsymbol{\theta}_j^{(2)}\Big \|_{\infty} \leq \dfrac{C}{N}$. Thus $\mathcal{G}_{k,\w}\subseteq\mathcal{F}_{\w}$.\\
    Conversely let $\boldsymbol{f}(\x)\in\mathcal{F}_{\w}$, then $\boldsymbol{f}(\x) = \displaystyle\sum\limits_{j=1}^N \boldsymbol{\alpha_j} \,\phi(\mathbf{x}_k^T\boldsymbol{\omega}_j)$. Choose $\boldsymbol{\theta}^{(2)}_j = \boldsymbol{\alpha}_j$ and $\boldsymbol{\theta}^{(1)}_{kj} = [0,\cdots,0]$.\\
    As $\cos(\boldsymbol{\theta}_{kj}^{(1)})\boldsymbol{\theta}_j^{(2)} = \boldsymbol{\alpha}_j$, thus $\boldsymbol{f}(\x) = \displaystyle\sum\limits_{j=1}^N \cos(\boldsymbol{\theta}_{kj}^{(1)})\boldsymbol{\theta}_j^{(2)}\phi(\mathbf{x}^T\boldsymbol{\omega}_j)$ and $\|\boldsymbol{\theta}_j^{(2)}\|_{\infty} = \|\boldsymbol{\alpha}_j\|_{\infty}\leq\dfrac{C}{N}$. \\
    Hence $\f(\x)\in\mathcal{G}(k,\w)$ and $\mathcal{F}_{\w}\subseteq\mathcal{G}_{k,\w}$.
\end{proof}

In the next Lemma stated below, we extend results from \cite{rahimi2007random,rahimi2008weighted,rahimi2008uniform}
to find approximation error bounds for vector-valued functions.

\begin{lemma}\label{th:rfm}
     Let $X\subset \R^d$ denote the training dataset and suppose $q$ is a measure on X, and $\boldsymbol{f}^*$ a function in $\mathcal{F}_{\rho}$ where 
     \begin{equation*}
         \mathcal{F}_{\rho}=\left\{ \boldsymbol{f}(\mathbf{x}) = \int_{\Omega} \boldsymbol{\alpha}(\boldsymbol{\omega}) \phi(\mathbf{x};\boldsymbol{\omega}) \, d\boldsymbol{\omega} \ \Bigg| \  \| \boldsymbol{\alpha}(\boldsymbol{\omega})\|_{\infty}\leq C{\rho(\boldsymbol{\omega})}\right\}.
     \end{equation*}  
     If $[\boldsymbol{\omega}_j]_{j\in [N]}$ are drawn iid from $\rho$, then for $\delta > 0$, with probability at least $1 - \delta$ over $[\boldsymbol{\omega}_j]_{j\in [N]}$, there exists a function $\boldsymbol{f}^\sharp\in\mathcal{F}_{\boldsymbol{\omega}}$ such that
    \begin{equation}\label{eq:bound}
        \|\boldsymbol{f}^\sharp-\boldsymbol{f}^*\|_2\leq \dfrac{C\sqrt{d}}{\sqrt{N}}\left(1+\sqrt{2\log\dfrac{1}{\delta}}\right),
    \end{equation}
    where \begin{equation}\label{eq:fw}
    \mathcal{F}_{\boldsymbol{\omega}} =\left\{\boldsymbol{f}(\x) = \sum\limits_{j=1}^N \boldsymbol{\alpha_j} \,\phi(\mathbf{x}^T\boldsymbol{\omega}_j)\Bigg| \|\boldsymbol{\alpha_j}\|_{\infty}\leq \frac{C}{N}\right\}.\end{equation}
\end{lemma}
\begin{proof}
We follow the proof technique described in \cite{rahimi2008weighted}. 
As $\f^*\in\mathcal{F}_{\rho}$, then $\f^* (\x)=\displaystyle\int_{\boldsymbol{\Omega}}\boldsymbol{\alpha}(\boldsymbol{\omega})\phi(\x;\boldsymbol{\omega})d\boldsymbol{\omega}$. 
   Construct $\f_k = \boldsymbol{\beta}_k\phi(\cdot;\boldsymbol{\omega}_k)$, $k=1,\cdots,N$ such that $\boldsymbol{\beta}_k = \dfrac{\boldsymbol{\alpha}(\boldsymbol{\omega}_k)}{\rho(\boldsymbol{\omega}_k)} = \dfrac{1}{\rho(\boldsymbol{\omega}_k)}\begin{bmatrix}
       \alpha_1(\boldsymbol{\omega}_k)\\
       \vdots\\
       \alpha_d (\boldsymbol{\omega}_k)
   \end{bmatrix}$. \\
   Note that $\mathbb{E}_{\boldsymbol{\omega}}(\f_k) = \displaystyle\int_{\boldsymbol{\omega}}\dfrac{\boldsymbol{\alpha}(\boldsymbol{\omega}_k)}{\rho(\boldsymbol{\omega}_k)}\phi(\x;\boldsymbol{\omega})\rho(\boldsymbol{\omega}_k)d\boldsymbol{\omega} = \f^*$.\\
   Define the sample average of these functions as $\f^{\sharp}(\x) = \displaystyle\sum\limits_{k=1}^N\frac{\boldsymbol{\beta}_k}{N}\phi(\x;\boldsymbol{\omega}_k)$. 
   \\
   As $\Big\|\dfrac{\boldsymbol{\beta}_k}{N}\Big\|_{\infty}\leq \dfrac{C}{N}$, thus $\f^{\sharp}\in\mathcal{F}_{\w}$. Also note that $\|\boldsymbol{\beta}_k\phi(\cdot;\w_k)\|_2\leq\sqrt{d}\|\|\boldsymbol{\beta}_k\phi(\cdot;\w_k)\|_{\infty}\leq\sqrt{d}C$.\\
   In order to getthe desired result, we use McDiarmid's inequality. Define a scaler function on $F = \{\f_1,\cdots,\f_N\}$ as $g(F) = \|\f^\sharp - \mathbb{E}_{F}\f^\sharp\|_2$. We claim that the function $g$ is stable under perturbation of its $i$th argument.\\
   Define $\Tilde{F}=\{\f_1,\cdots,\Tilde{\f}_i,\cdots,\f_N\}$ i.e., $\Tilde{F}$ differs from $F$ only in its $i$th element. Then
   \begin{equation}
       \Big|g(F) - g(\Tilde{F})\Big| = \Big|\Big\|\f^\sharp - \mathbb{E}_{F}\f^\sharp\Big\|_2 - \Big\|\Tilde{\f}^\sharp - \mathbb{E}_{\Tilde{F}}\f^\sharp\Big\|_2\Big|\leq \Big\|\f^\sharp -\Tilde{\f}^\sharp\Big\|_2
   \end{equation}
   where the above inequality is obtained from triangle inequality. Further,
   \begin{equation}
       \Big\|\f^\sharp -\Tilde{\f}^\sharp\Big\|_2 = \dfrac{1}{N}\Big\|\f_i - \Tilde{\f}_i\Big\|_2\leq \Big\|(\boldsymbol{\beta}_i - \Tilde{\boldsymbol{\beta}}_i)\phi(\cdot;\w)\Big\|_2\leq\dfrac{\sqrt{d}C}{N}.
   \end{equation}
   Thus $\mathbb{E}[g(F)^2] = \mathbb{E}\left[\Big\|\f^\sharp - \mathbb{E}_{F}\f^\sharp\Big\|_2^2\right] = \dfrac{1}{N}\left[\mathbb{E}\left[\Bigg\|\displaystyle\sum\limits_{k=1}^N\f_k\Bigg\|_2^2\right] - \Bigg\|\mathbb{E}\left[\displaystyle\sum\limits_{k=1}^N \f_k\right]\Bigg\|_2^2\right]$.\\
   Since $\|\f_k\|_2\leq\sqrt{d}C$, using Jensen's inequality and above result we get
   \begin{equation}
       \mathbb{E}[g(F)]\leq\sqrt{\mathbb{E}(g^2(F))}\leq \dfrac{\sqrt{d}C}{\sqrt{N}}.
   \end{equation}
   Finally the required bounds can be obtained by combining above result and McDiarmid's inequality.
\end{proof}

Using the above-stated Lemmas and  results given in \cite{chen2022sampling}, we derive our main theorem. Specifically, we quantify the total variation between the distribution learned by our model and the true data distribution.

\begin{theorem}\label{th:main}
For a given probability density $q(\x_0)$ on $\mathbb{R}^d$ suppose the following conditions hold:
\begin{enumerate}
    \item For all $t \geq 0$, the score $\nabla \log q(\x(t))$ is $L-$Lipschitz.
    \item The second moment of $q(\x_0)$ is finite i.e.,  $m_2^2 = E_{q(\x_0)}[\|.\|^2]<\infty$.
\end{enumerate}
Let $p_{\theta}(\x_0,0)$ be the sample generated by DRFM after $K$ timesteps at terminal time $T$. Then for the SDE formulation of the DDPM algorithm, if the step size $h := T/K$ satisfies $h \lesssim 1/L$, where $L \geq 1$. Then,
\begin{align}\label{eq:tv}
    TV(p_{\theta}(\x_0,0),q(\x_0))&\lesssim \sqrt{KL(q(\x_0)\|\gamma)}\exp(-T) + (L\sqrt{dh} + Lm_2h)\sqrt{T} + \dfrac{C_2\sqrt{TKd}}{\sqrt{N}}\left(1+\sqrt{2\log\dfrac{1}{\delta}}\right)
\end{align}
where $C_2 \geq \max\limits_{1\leq i\leq N, 1\leq j \leq d}\left|
\boldsymbol{\theta}_{ij}^{(2)}\right|$ and $\gamma$ is the p.d.f. of the multivariate normal distribution with mean vector $\mathbf{0}$ and covariance matrix $\mathbf{I}$. 
 \end{theorem}

The error bound given in Eq. \eqref{eq:tv} consists of three types of errors: (i) convergence error of the forward process; (ii) discretization error of the associated SDE with step size $h>0$; and (iii) score estimation error, respectively. Note that the first two terms are independent of the model architecture. While for most models score estimation is difficult, our main contribution involves quantifying that error which gives us an estimate of the number of parameters are needed for the third error to become negligible. We can combine the proof of Lemma \ref{lemma:equality}, \ref{th:rfm} and results of Theorem 1 from \cite{chen2022sampling} to get the required bounds.

\section{Experimental Results}\label{sec:results}
In order to validate our findings, we train our proposed model on both image and audio data. We evaluate the validity of our model on two tasks: (i) generating data from noise and; (ii) denoising data corrupted with gaussian noise. For audio data, we use two music samples corresponding to flute and guitar. The experiment images was done by taking one hundred images of the class ``dress" and ``shoes" separately from fashion MNIST dataset. 

We compare our method with: (1) fully connected version of DRFM denoted by NN where we train $[\mathbf{W},\mathbf{b},\boldsymbol{\theta}^{(1)},\boldsymbol{\theta}^{(2)}]$ while preserving the number of trainable parameters; (2) classical random feature approach with only $\boldsymbol{\theta}^{(2)}$ being trainable denoted by RF. The details of the implementation of all the experiments and their results are described in the sections below.

\subsection{Results on Fashion-MNIST data}
We create the image dataset for our experiments by considering 100 images of size $28\times 28$ taken from a particular class of fashion MNIST dataset. The model is trained with 80000 semi-random features with $100$ equally spaced timesteps between 0.0001 and 0.02 trained for 30000 epochs. For generating images, we generated fifteen samples from randomly sampled noise. In our second task we give fifteen images from the same class (but not in the training set) corrupted with noise and test if our model can denoise the image.
Results from Figure \ref{fig:generated} demonstrate that our method learns the overall features of the input distribution. Although the model is trained with a very small number of training data (only one hundred images) and timesteps (one hundred timesteps), we can see that the samples generated from pure noise have already started to resemble the input distribution of dresses.
\begin{figure}
    \centering
    \includegraphics[scale = 0.25]{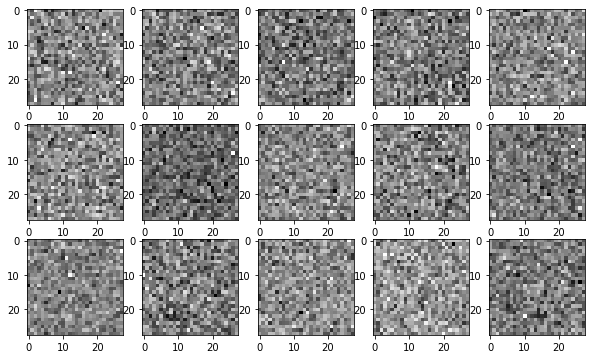}
    \includegraphics[scale =0.25]{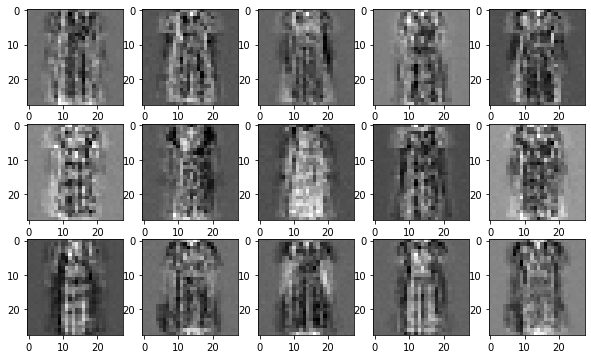}\\
    \includegraphics[scale = 0.25]{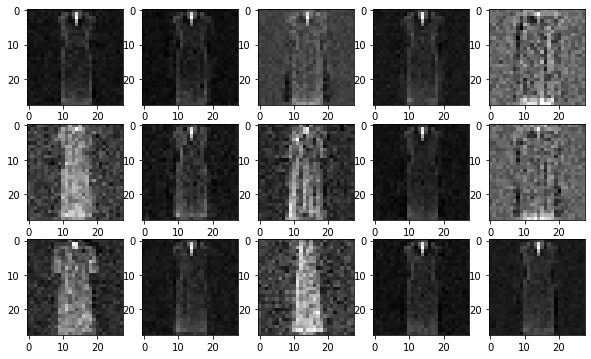}
    \includegraphics[scale = 0.25]{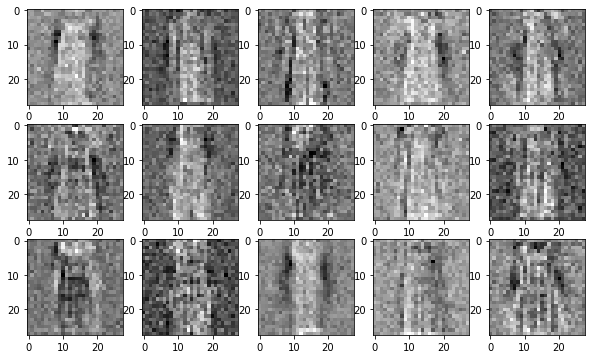}
    \caption{Figures generated from random noise when trained on 100 ``dress" images.  Samples from: Gaussian noise (top left), DRFM (top right), fully connected network (bottom left) and random feature model (bottom right). }
    \label{fig:generated}
\end{figure}
For NN model we note that most of the generated samples are the same with a dark shadow while for RF model, the generated images are very noisy and barely recognizable.

We also test our models ability to remove noise from images. We take fifteen random images of ``dress" not in the training data and corrupt it with $20\%$ noise. The proposed model is used for denoising. In Figure \ref{fig:denoised} we can see that the model can recover a denoised image which is in fact better than the results obtained when sampling from pure noise. The generated images for most of the noised data look very similar to the original ones in terms of their shape and size. The NN model performs quite well for most of the images, however for a few cases it fails to denoise anything and the final image looks similar to the noisy input. RF model fails to denoise and the resultant images still contain noise.

\begin{figure}[h!]
    \centering
    \includegraphics[scale = 0.25]{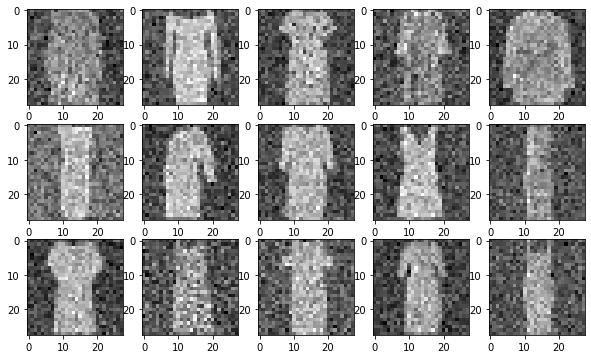}
    \includegraphics[scale = 0.25]{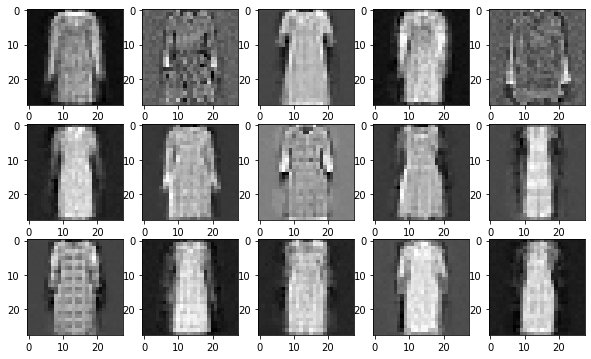}\\
    \includegraphics[scale = 0.25]{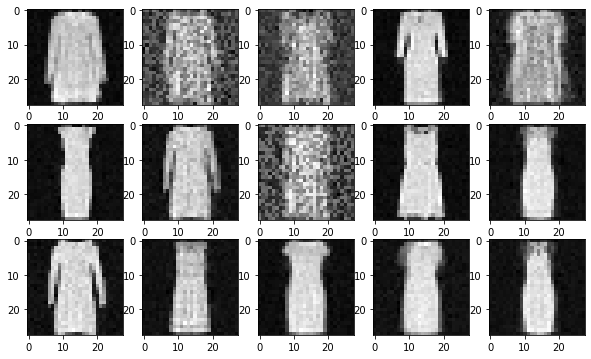}
     \includegraphics[scale = 0.25]{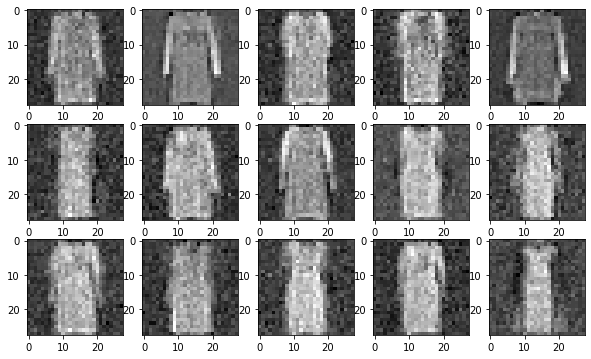}
    \caption{Figures denoised from corrupted ``dress" images. Corrupted images (top left); denoised images by DRFM (top right), fully connected network (bottom left) and random feature model (bottom right).}
    \label{fig:denoised}
\end{figure}

In order to check the effect of the number of timesteps on the sampling power of DRFM, we also run our model using 1000 timesteps between 0.0001 and 0.02. The images generated/denoised are given in Figure \ref{fig:DRFM1000}. Samples generated from noise seem to improve with the increase in number of timesteps. However, for the task of denoising, the results are better with 100 timesteps. The improvement in sample quality with an increased number of timesteps for generating data is expected as more number of reverse steps would be required to generate a point in the input distribution than for the task of denoising.

\begin{figure}[h]
    \centering
    \includegraphics[scale = 0.25]{random_noise_dress.png}
    \includegraphics[scale = 0.25]{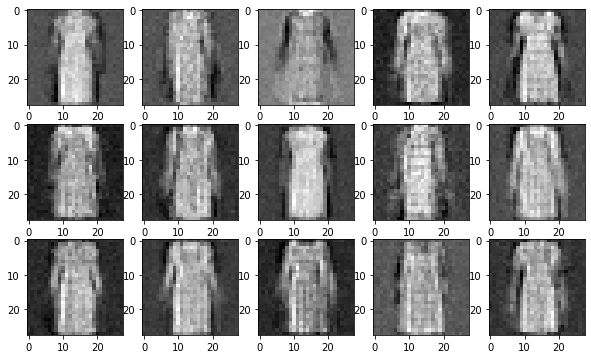}\\
     \includegraphics[scale = 0.25]{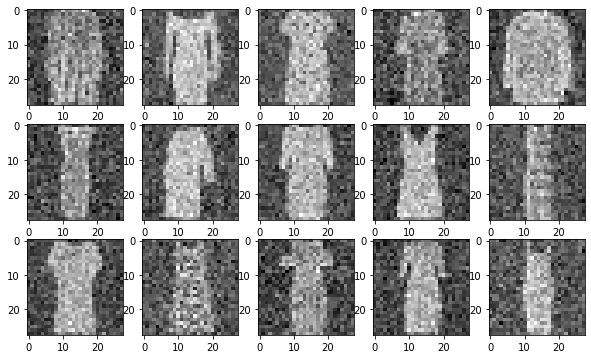}
     \includegraphics[scale = 0.25]{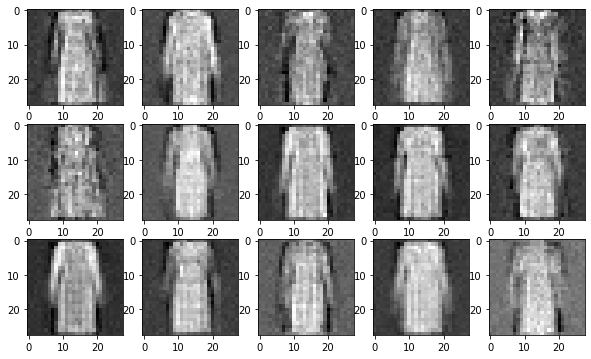}
    \caption{Generated and denoised images trained on 100 ``dress" images with 1000 timesteps. Top row depicts samples generated from Gaussian noise. Bottom row depicts denoised images. }
    \label{fig:DRFM1000}
\end{figure}

We also conduct experiment of a different class of data with the same hyperparameters as discussed above with 100 timesteps. This time we select the class ``shoes" and test our model's performance. The conclusions drawn supports the claims we made with the previous class of data where DRFM can denoise images well while only learning to generate basic features of the model class when generating from pure noise. The plots for the above are depicted Figure \ref{fig:shoe}.

\begin{figure}[h!]
    \centering
    \includegraphics[scale=0.25]{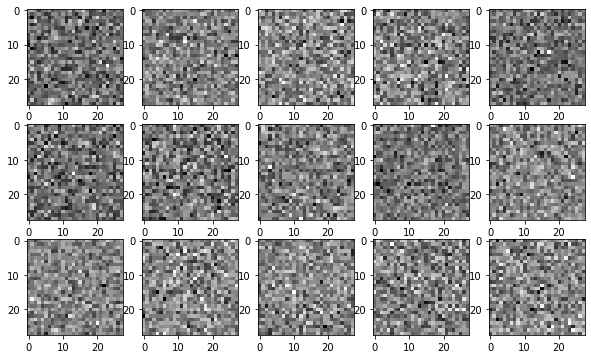} 
    \includegraphics[scale=0.25]{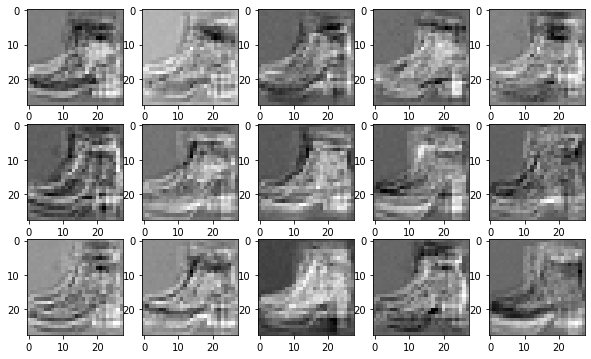}\\
    \includegraphics[scale=0.25]{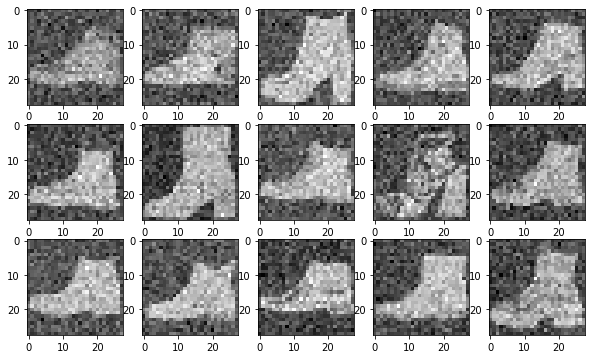} 
    \includegraphics[scale=0.25]{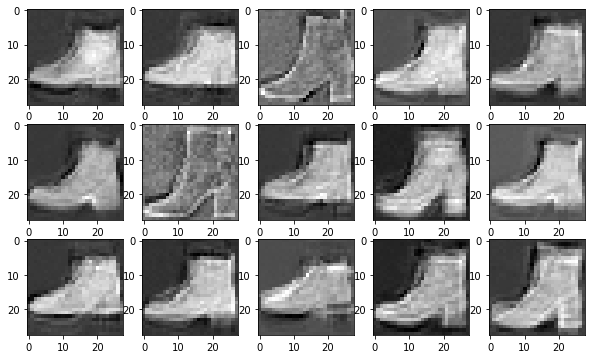}
    \caption{Samples generated from noise and noisy images when trained on 100 images of ``shoes". Top row depicts samples generated from Gaussian noise. Bottom row depicts denoised images.}
    \label{fig:shoe}
\end{figure}

\subsection{Results on Audio Data}
Our second experiment involves learning to generate and denoise audio signals. We use one data sample each taken from two different instruments, namely guitar and flute. There are a total of 5560 points for each instrument piece. We train our model using 15000 random features with 100 timesteps taken between 0.0001 and 0.02 for 30,000 epochs. 
The samples are generated from pure noise using the trained model to remove noise for each reverse diffusion step. We also test if our model is capable of denoising a signal when it is not a part of the training set explicitly (but similar to it). For that, we use a validation data point containing samples from a music piece when both guitar and flute are played simultaneously. We plot the samples generated from pure Gaussian noise in Figure \ref{fig:genaud}.
The plots in Figure \ref{fig:genaud} demonstrate the potential of DRFM to generate signals not a part of the original training data. Figure \ref{fig:genaud}(b) shows that there is no advantage of using a NN model since the results are much worse. The network does not learn anything and the signal generated is just another noise. On the other hand, our proposed model DRFM generates signals that are similar to the original two signals used to train the data.
\begin{figure}[h!]
\centering

\begin{subfigure}[b]{0.3\textwidth}
    \centering
    \includegraphics[scale = 0.45]{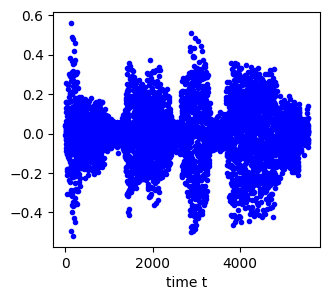}
    \caption{Final sample (DRFM)}
    \end{subfigure}
    \centering
\begin{subfigure}[b]{0.3\textwidth}
    \centering
    \includegraphics[scale = 0.45]{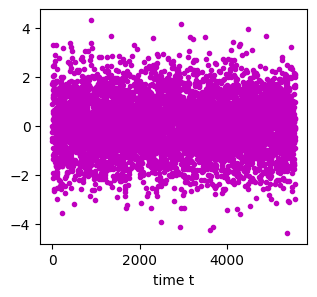}
    \caption{Final sample (NN)}
    \end{subfigure}
    \begin{subfigure}[b]{0.3\textwidth}
    \centering
    \includegraphics[scale = 0.45]{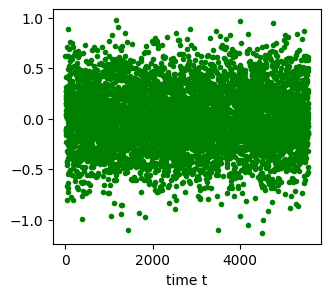}
    \caption{Final sample (RF)}
    \end{subfigure}
    
    \caption{Results using (a)DRFM; (b) NN; (c) random feature model (RF). }
    \label{fig:genaud}
\end{figure}

Figure \ref{fig:denoisedaud} shows that when our trained model is applied to a validation data point, it can successfully recover and denoise the signal. This is more evident when the signal is played as an audio file. There are however some extra elements that get added while recovering due to the presence of noise which is a common effect of using diffusion models for denoising.

\begin{figure}[h!]
\centering

\begin{subfigure}[b]{0.3\textwidth}
    \centering
    \includegraphics[scale = 0.45]{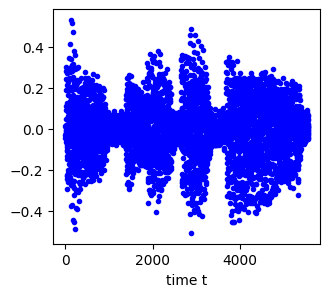}
    \caption{Denoised signal (DRFM)}
    \end{subfigure}
\begin{subfigure}[b]{0.3\textwidth}
    \centering
   \includegraphics[scale = 0.45]{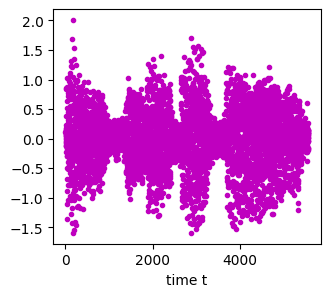}
    \caption{Denoised audio signal (NN)}
    \end{subfigure}
\begin{subfigure}[b]{0.3\textwidth}
    \centering
   \includegraphics[scale = 0.45]{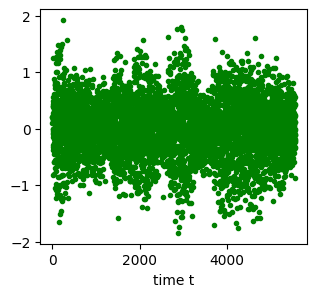}
    \caption{Denoised audio signal (RF)}
    \end{subfigure}

   \caption{Results using (a) DRFM; (b) NN; (c) random feature model (RF). }
    \label{fig:denoisedaud}
\end{figure}

\section{Conclusion}\label{sec:conc}
In this paper, We proposed a diffusion random feature model. It was shown that with our network architecture, the model becomes interpretable helping us to find theoretical upper bounds on the samples generated by DRFM with respect to the input distribution. We validated our findings using numerical experiments on audio and a small subset of the fashion MNIST dataset. Our findings indicated the power of our method to learn the process of generating data from as few as one hundred training samples and one hundred timesteps. Comparisons with a fully connected network (when all layers are trainable) and random features method (all but the last layer is fixed i.e., only $\boldsymbol{\theta}_2$ is trainable) highlighted the advantages of our model which performed better than both. 

Our proposed model combining the idea of random features with diffusion models opens up potential directions worthy of further exploration into the power of random features in developing models which are interpretable and suited to complex tasks. Further direction of this work involves extending DRFM beyond its shallow nature into deeper architecture to avoid the curse of dimensionality with respect to the number of features required for approximation.

\section*{Acknowledgements}
E.S. and G.T. were supported in part by NSERC RGPIN 50503-10842.

\medskip

\bibliographystyle{abbrv}
\bibliography{references}{}

\end{document}